\documentclass[12pt, final]{l4dc2022}

\title[Safe RL with Chance-constrained MPC]{Safe Reinforcement Learning with Chance-constrained Model Predictive Control}
\usepackage{times}

\author{%
 \Name{Samuel Pfrommer} \Email{sam.pfrommer@berkeley.edu}\\
 \vspace{-0.5cm}
 \AND
 \Name{Tanmay Gautam} \Email{tgautam23@berkeley.edu}\\
 \vspace{-0.5cm}
 \AND
 \Name{Alec Zhou} \Email{auloehcz@berkeley.edu}\\
 \vspace{-0.51cm}
 \AND
 \Name{Somayeh Sojoudi} \Email{sojoudi@berkeley.edu}\\
 \addr Department of Electrical Engineering and Computer Sciences at UC Berkeley
}

\RequirePackage{algorithm}
\RequirePackage{algorithmic}
\usepackage{microtype}
\usepackage{graphicx}
\usepackage[]{caption}
\usepackage{amsmath}
\usepackage{amsfonts}
\usepackage{amssymb}
\usepackage{mathtools}
\usepackage{twoopt}
\usepackage{booktabs} 

\usepackage{tikz}
\usetikzlibrary{shapes,arrows,automata}

\usepackage{hyperref}

\newcommand{\R}{{\mathbb R}}
\newcommand{\normal}{{\mathcal{N}}}
\newcommand{\KL}{{\mathrm{KL}}}
\DeclareMathOperator*{\E}{\mathbb{E}}
\newcommand{\Tr}{{\mathrm{Tr}}}
\newcommand{\T}{^\intercal}
\newcommand{\ones}[1]{{\mathbf{1}_{#1}}}

\newcommand{\defeq}{\vcentcolon=}

\DeclareMathOperator*{\argmin}{arg\,min}

\newcommandtwoopt{\st}[2][t][]{{s_{#1}^{#2}}}
\newcommandtwoopt{\ac}[2][t][]{{a_{#1}^{#2}}}
\newcommand{\re}{r}

\newcommand{\stprime}{{\st[]'}}

\newcommand{\stspace}{{\mathcal{X}}}
\newcommand{\acspace}{{\mathcal{A}}}

\newcommand{\stsafe}{{\mathcal{S}}}
\newcommand{\stsafeterm}{{\mathcal{S}_T}}
\newcommand{\acset}{{\mathcal{A}}}

\newcommand{\pol}{{\pi_{\theta}}}
\newcommand{\polsafe}{{\pi_{\theta}^{\mathrm{safe}}}}
\newcommand{\polstar}{{\pi_{\theta*}}}
\newcommand{\polsafestar}{{\pi_{\theta^*}^{\mathrm{safe}}}}
\newcommand{\poltarget}{{\pi^*}}

\newcommand{\Pol}{{\Pi}}

\newcommand{\stmu}[1][t]{{\mu_{#1}^{\st[]}}}
\newcommand{\acmu}[1][t]{{\mu_{#1}^{\ac[]}}}

\newcommand{\acsig}[1][t]{{\Sigma_{#1}^{\ac[]}}}
\newcommand{\stsigbar}[1][t]{{\overline{\Sigma}_{#1}^{\st[]}}}
\newcommand{\acsigbar}[1][t]{{\overline{\Sigma}_{#1}^{\ac[]}}}

\newtheorem{assumption}{Assumption}

\begin{document}
\maketitle
\begin{abstract}
Real-world reinforcement learning (RL) problems often demand that agents behave safely by obeying a set of designed constraints. We address the challenge of safe RL by coupling a \textit{safety guide} based on model predictive control (MPC) with a modified policy gradient framework in a linear setting with continuous actions. The guide enforces safe operation of the system by embedding safety requirements as chance constraints in the MPC formulation. The policy gradient training step then includes a safety penalty which trains the base policy to behave safely. We show theoretically that this penalty allows for a provably safe optimal base policy and illustrate our method with a simulated linearized quadrotor experiment.
\end{abstract}

\section{Introduction}
Reinforcement learning has been extensively studied in the context of closed environments, where it has gained popularity for its success in mastering games such as Atari and Go \citep{Sutton+Barto:1998, Mnih2015HumanlevelCT, silver2017mastering}. A pressing need to deploy autonomous agents in the physical world has introduced a new challenge: agents must be able to interact with their environments in a safe and comprehensible manner. This is especially critical in industrial settings \citep{dalal2018safe}.

For safety-critical tasks, the trial-and-error nature of exploration in RL often prevents agent deployment in the real world during training, motivating the use of simulators. However, when dealing with complex environments, simulators may fail to sufficiently model the complexity of the environment \citep{Achiam2019BenchmarkingSE}. Furthermore, reward functions may be unknown a priori, making learning in simulation impossible. This is where methods that guarantee safe exploration during training offer a substantial advantage.

Our work employs policy gradients and model predictive control (MPC) as its primary building blocks to address the safe RL problem. Policy gradient methods learn a parameterized policy to maximize long-term expected rewards using gradient ascent and play a central role in reinforcement learning due to their ability to handle stochasticity, superior convergence properties and training stability, and efficacy in high-dimensional action spaces \citep{Sutton+Barto:1998}. This family of algorithms is also \textit{model-free}, relying solely on reward signals from the environment without modeling any dynamics. Policy gradient variations have since proliferated under the deep learning paradigm, notably including ``natural'' policy gradients and actor-critic methods in addition to techniques such as experience replay and importance sampling for better sample efficiency \citep{PETERS2008682, wang2017sample}.

Model predictive control is a flexible optimal control framework that has seen successes across a wide variety of settings, including process control in chemical plants and oil refineries, power electronics and power system balancing, autonomous vehicles and drones, and building control \citep{qin2003survey, rawlings2009model}. It is \textit{model-based}, requiring the system dynamics to be identified either a priori or through learning \citep{koller2018learningbased}. Its interpretability lends itself to robust extensions, where system uncertainties and disturbances can be incorporated to probabilistically guarantee agent safety \citep{koller2018learningbased}. 
\subsection{Related Work}
    Safety filters are the closest line of work to our proposed algorithm \citep{wabersich2019linear,wabersich2021predictive}. This is a decoupled method that takes sampled actions from any base policy and uses an MPC controller as the ``safety filter'' to correct unsafe behaviors. However, these two components function independently, which may lead to conflicting and potentially oscillatory behaviour between the MPC and RL objectives. The computationally taxing safety filter must also be used at both training and test times, making the technique ill-suited for real-world deployment on constrained hardware. \cite{end-to-endsaferl} proposes a related framework that combines model-free RL algorithms with control barrier functions to guarantee safety during training. While this approach accommodates model uncertainty and learns the dynamics online, it is decoupled in a manner similar to safety filters and retains the same drawbacks. \cite{Wagener2021SafeRL} describes SAILR - an alternative intervention-based approach that utilizes advantage functions to learn a safe policy during training. While empirically the authors demonstrate that this algorithm outperforms other safe RL methods, it is still shown to occasionally violate safety constraints during training. 
    
    Constrained reinforcement learning (CRL) aims to formalize the reliability and safety requirements of an agent by encoding these explicitly as constraints within the RL optimization problem. \cite{DBLP:journals/corr/AchiamHTA17} proposes a trust-region based policy search algorithm for CRL with guarantees, under some policy regularity assumptions, that the policy stays within the constraints in expectation. This approach cannot be used in applications where safety must be ensured at all visited states. \cite{dalal2018safe} addresses the CRL problem by adding a safety
    layer to the policy that analytically solves an action correction formulation for each state. While this approach guarantees constraint satisfaction, it does not yield a safe policy at the end of training. In \cite{tessler2018reward}, the constraints are embedded as a penalty signal into the reward function, guiding the policy towards a constraint satisfying solution. Similar to \cite{DBLP:journals/corr/AchiamHTA17}, safety is not ensured at each state. 
    
    Model-based RL methods generally offer higher sample efficiency than their model-free counterparts and can be applied in safety-critical settings with more interpretable safety constraints. This area of work includes learning-based robust MPC \citep{koller2018learningbased}. \cite{berkenkamp2017safe} proposes an algorithm that considers safety in terms of Lyapunov stability guarantees. More specifically, the approach demonstrates how, starting from an initial safe policy, the safe region of attraction can be expanded by collecting data within the safe region and adapting the policy. 
    
    Imitation learning attempts to learn a policy by direct supervision from expert demonstration. This approach is frequently plagued by distribution mismatch and compounding errors. Dataset Aggregation (DAgger) is an iterative method used to mitigate these drawbacks by reducing the distribution mismatch \citep{journals/jmlr/RossGB11}. In \citep{Menda2019EnsembleDAggerAB}, the authors extend DAgger to EnsembleDAgger, which addresses the challenge of safe exploration by quantifying the confidence of the learned policy. It does this by using an ensemble of neural networks to estimate the variance of the action proposed by the learned policy at a particular state. While showing solid empirical performance, EnsembleDAgger lacks formal safety guarantees.
\subsection{Paper Contributions}
Our approach wraps a policy gradient \textit{base policy} with an MPC-based \textit{safety guide} that corrects any potentially unsafe actions. The base policy learns to optimize the agent's long-term behaviour, while the MPC component accounts for state-space safety constraints. By optimizing over an action distribution in the safety guide, we show that adding a safety penalty to the policy gradient loss allows for a provably safe optimal base policy. This resolves tension between the base policy and the safety guide and permits the removal of the computationally expensive safety guide after training.
\section{Background}
\subsection{Notation}
Throughout this work, we let $\st \in \stspace$, $\ac \in \acspace$, and $\re(\st, \ac ) \in \R$ refer to the state, action, and reward at time $t$. A sequence of states and actions is termed a trajectory and denoted by $\tau$, and the sum of rewards over a trajectory is denoted $\re(\tau)$. We focus on the setting where $\stspace \subseteq \R^n$ and $\acspace \subseteq \R^m$. Since our action space is continuous, we represent a stochastic policy as $\pi : \stspace \rightarrow \normal(\acspace)$, where $\normal(\acspace)$ is a Gaussian distribution over actions. More specifically, we can write $\pi(~ \cdot \mid \st[]) = \normal(\mu(\st[]), \Sigma(\st[]))$ for some Gaussian mean $\mu(\st[])$ and covariance $\Sigma(\st[])$. The space of such policies is denoted as $\Pol$. When such a policy is parameterized by a vector $\theta$, we use the notation $\pol$. With some abuse of notation, we write $\tau \sim \pi$ to denote sampling a trajectory from the policy $\pi$; similarly, $(\st[], \ac[]) \sim \pi$ denotes sampling a state $\st[]$ from the stationary distribution induced by $\pi$ and then sampling $\ac[]$ from $\pi(~ \cdot \mid \st[])$. Furthermore, $\|\cdot\|_p$ denotes the $\ell_p$-norm within $\R^n$. The symbol $\ones{n}$ defines an $n$-dimensional column vector of ones, and $\Tr(A)$ denotes the trace of the matrix $A$. $\E_{p(x)}[\cdot]$ is the expectation operator with respect to the probability distribution $p(x)$.

\subsection{Policy Gradient}
Policy gradient methods attempt to find the optimal parameters $\theta^*$ for the objective
\begin{align}\label{eq:RLoptimization}
    \max_{\theta} J(\pol), \quad ~ \quad ~ \quad J(\pol) = \E_{\tau \sim \pol}\bigg[\sum_{t=0}^M \gamma^t \re(\st, \ac)\bigg].
\end{align}
The vanilla policy gradient approach performs gradient ascent to maximize this objective \citep{Williams:92}. The gradient can be approximated with the Monte-Carlo estimator 
\begin{align}\label{eq:vanillapolicygrad}
    \nabla_{\theta} J(\pol) \approx \frac{1}{N}\sum_{j=1}^N\sum_{t=0}^M \nabla_{\theta}\log \pol(\ac[t][j] \mid \st[t][j]) \sum_{t'=t}^M \gamma^{t' - t} \re(\st[t'][j], \ac[t'][j]),
\end{align}
with $0 \leq \gamma < 1$ a discount factor. While many variance-reduction techniques can be used to improve \eqref{eq:vanillapolicygrad}, for simplicity of exposition we employ this basic formulation.
    

\subsection{Model Predictive Control}
Model predictive control is a purely optimization-based planning framework. Given a dynamics model and a set of state and action constraints (safety requirements, physical limitations, etc.), the finite-horizon MPC problem computes the near-optimal open-loop action sequence that minimizes a specified cumulative cost function. The first of these actions is executed, and the entire optimization repeats on the next time step. While the MPC framework offers concreteness in its constraints, it requires a pre-specified reward function and is incapable of forming reward-maximizing plans beyond its horizon.
\subsection{Problem Setting}
We represent the environment dynamics as a known linear time-invariant system
\begin{align}\label{eq:lineardynamics}
\st[t+1] = A \st + B \ac,
\end{align}
with initial state $\st[0]$, dynamics matrix $A\in\R^{n\times n}$, and input matrix $B\in\R^{n\times m}$. The safety requirements are captured by a polyhedral state safe set $\stsafe \subset \stspace$. The goal is to learn a policy which maximizes the cumulative reward signal $\re$ while ensuring that the exploration during training is safe at all times, i.e. $\st\in\stsafe$ for all $t$.

\section{Method}



A high-level overview of our method combining policy gradient learning and model predictive control is displayed in Figure~\ref{fig:approachschematic}. We first outline the construction of the \textit{safety guide}, which solves a chance-constrained MPC optimization to enforce the safety of actions proposed by the underlying base policy. This allows for guaranteed safety during training time with arbitrarily high probability. Section~\ref{sec:policy} discusses how the safety guide is incorporated into the overarching policy optimization.

\subsection{Safety Guide Design}
\label{sec:guide}
The safety guide solves a convex MPC problem for each time step during training to ensure system safety. This safety guide is not needed later at test time, which is justified theoretically in Section~\ref{sec:analysis}. We begin by making the following assumption.
\begin{assumption} \label{ass:stable}
    There exists a polyhedral terminal safe set $\stsafeterm \subset \stsafe \subset \stspace$ that is \textup{invariant}, meaning that for any state $\st[] \in \stsafeterm$, there exists a sequence of control inputs that keep the system in $\stsafeterm$ for all subsequent time steps.
\end{assumption}
The construction of invariant sets has a well-established theory due to its applications in systems and control. For linear systems, several recursive algorithms have been proposed to construct polyhedral invariant sets \citep{83532, 1470058}, with nonlinear systems considered in \cite{7084969, korda2013convex}.

\begin{figure*}
  \centering
  \begin{minipage}[t]{0.4\textwidth}
    \centering
    \raisebox{-\height}{\resizebox{0.9\textwidth}{0.7\height}{
\begin{tikzpicture}[auto, node distance=4cm,>=latex']
    \tikzstyle{block} = [draw, fill=white!20, rectangle, 
    minimum height=2em, minimum width=6em]
    \node[block](base){Base Policy};
    \node[block, below of=base, yshift=1cm](safe){Safety Guide};
    \node[block, below of=safe, yshift=1cm](env){Environment};
    \node[block, right of=env](obj){$J^p(\pol)$};
    \node[block, right of=safe](pen){Safety Penalty};
    \draw [draw,->] (base) -- node (t2){$\pi_{\theta}(\cdot \vert s_t)$} (safe);
    \draw[draw, ->] (env)--([shift={(-10mm,0mm)}]env.west)|- node[above right] {$s_{t+1}$} ([shift={(-10mm,0mm)}]base.west)-- (base);
    \draw [draw,->] (safe) -- node (t1){$\pi_{\theta}^{\textrm{safe}}(\cdot \vert s_t)$} (env);
    \draw [draw,->] ([shift={(0mm,5mm)}]safe.north) -| ([shift={(18mm,5mm)}]safe.north) |- ([shift={(0mm,1.55mm)}]pen.west);
    \draw [draw,->] ([shift={(0mm,-5mm)}]safe.south) -| ([shift={(18mm,-5mm)}]safe.south) |- ([shift={(0mm,-1.55mm)}]pen.west);
    \draw [draw,->] (pen) -- (obj);
    \draw [draw,->] (env) -- node {$r(s_t, a_t)$} (obj);
    \draw [draw,->] (obj) -- ([shift={(5mm, 0mm)}]obj.east) -- ([shift={(45mm, 0mm)}]base.east) -- node {$\nabla J^p(\pol)$} (base);
    
\end{tikzpicture}}}
    \vspace{-0.2cm}%
    \caption{
   The training scheme. The base policy $\pol$ suggests a distribution over actions given $\st$. The safety guide potentially shifts this distribution to ensure safety and outputs the distribution $\polsafe( ~ \cdot \mid \st)$, from which the next action is sampled. The environment reward and a safety penalty on the distance between these two distributions are combined in the objective, whose gradient is approximated using Monte Carlo rollouts.
    }\label{fig:approachschematic} 
  \end{minipage}\hfill
\begin{minipage}[t]{0.58\textwidth}
\vspace{-0.32cm}
\begin{algorithm}[H]
   \caption{Safety guide}
   \label{alg:guide}
\begin{algorithmic}
   \STATE {\bfseries Input:} starting state $\stprime$, base policy mean $\acmu[\theta](\stprime)$ and covariance $\acsig[\theta](\stprime)$
   \STATE {\bfseries Parameters:} Planning horizon $H$, safety tolerance $\epsilon$, system matrices $A$ and $B$, state safe set $\stsafe$, safe terminal set $\stsafeterm$, feasible action set $\acset$
   \vspace{0.2cm}
   \STATE {\bfseries Solve} the convex optimization problem
   \vspace{-0.2cm}
    \begin{alignat*}{2}
        \argmin_{\substack{\acsigbar[0] \\ \acmu[0], \dots, \acmu[H] \\ \stmu[0], \dots, \stmu[H]}} \quad
        &\KL\left(\normal\big(\acmu[0], \acsigbar[0] \acsigbar[0]\T\big) ~ \lVert ~ \normal\big(\acmu[\theta](s'), \acsig[\theta](s')\big) \right) \\[-20pt]
        &\stmu[t+1] = A \stmu + B \acmu, ~ ~ 0 ~ {\leq} ~ t ~ {<} ~ H \\
        &\stsigbar \coloneqq A^t B \acsigbar[0], ~ ~ 0 ~ {\leq} ~ t ~ {<} ~ H \\
        &\Pr\big[\normal(\stmu, \stsigbar \stsigbar\T) \not \in \stsafe\big] < \epsilon, ~ ~ 0 ~ {\leq} ~ t ~ {<} ~ H \\
        &\Pr\big[\normal(\stmu[H], \stsigbar[H] \stsigbar[H]\T) \not \in \stsafeterm \big] < \epsilon, ~ ~ t ~ {=} ~ H \\
        &\acmu \in \acspace, ~ ~ 0 ~ {<} ~ t ~ {\leq} ~ H  \\
        &\stmu[0] = \stprime
    \end{alignat*}
    \STATE {\bfseries If} infeasible {\bfseries then} relax constraints and resolve
    \STATE {\bfseries Return} 
    $\polsafe( ~ \cdot \mid \stprime) = \normal\big(\acmu[0]^*, \acsigbar[0]^* {\acsigbar[0]^*}\T \big)$
\end{algorithmic}
\end{algorithm}
\end{minipage}
\vspace{-0.3cm}
\end{figure*}

Algorithm~\ref{alg:guide} specifies the safety guide optimization problem. Intuitively, the safety guide attempts to find an action distribution that is as close as possible to that outputted by the base policy, subject to safety constraints. Taking inspiration from techniques in the obstacle avoidance literature \citep{blackmore2011chance}, we formulate this in a chance-constrained model predictive fashion.

\noindent{\sc Variables}. The optimization variables consist of a sequence of state means $\stmu$ and open-loop control actions $\acmu$ over a planning horizon of length $H$, with the first action containing some uncertainty represented by $\acsigbar[0]$. The bar over any $\Sigma$ denotes that this matrix is related to the relevant covariance matrix via $\Sigma = \overline{\Sigma} ~ \overline{\Sigma} \T$. This decomposition allows for subsequent chance constraints to be expressed as closed-form convex constraints. Since we are interested in allowing the base policy to have as much freedom as possible, we avoid the additional conservatism that would result from incorporating uncertainty over future actions and allow these to be chosen deterministically.

\vspace{0.2cm}
\noindent{\sc Objective}. The safety guide objective minimizes the divergence between the base policy action distribution and the distribution of the MPC's first action. If the base policy distribution allows for subsequent actions that maintain safety, the objective vanishes and the returned safe distribution is the original distribution specified by the base policy. $\KL$ divergence is not symmetric; we choose this argument order to make the objective convex in the variables $\acmu[0]$ and $\acsigbar[0]$. To see this, consider the following form for the $\KL$ divergence, dropping references to $s'$ for notational convenience:
\begin{align*}
    &\KL\left(\normal\big(\acmu[0], \acsigbar[0] \acsigbar[0]\T\big) ~ \lVert ~ \normal\big(\acmu[\theta], \acsig[\theta]\big) \right) \\
    = ~ &\log \det \acsig[\theta] - \log \det \acsigbar[0] \acsigbar[0]\T - n 
    + \Tr\big((\acsig[\theta])^{-1} \acsigbar[0] \acsigbar[0] \T \big)
    + (\acmu[\theta] - \acmu[0])\T \big( \acsig[\theta] \big)^{-1} (\acmu[\theta] - \acmu[0]).
\end{align*}
Recall that symbols subscripted by $\theta$ are constants in the optimization, while symbols subscripted by $0$ are optimization variables. Therefore we disregard the constant terms $\log \det \acsig[\theta]$ and $-n$. Convexity of $- \log \det \acsigbar[0] \acsigbar[0]\T$ follows from multiplicative properties of the determinant and concavity of the $\log \det$ operator. For the remaining terms, we assume that $\acsig[\theta]$ is positive definite, a practically satisfied assumption. The fourth term $\Tr\big((\acsig[\theta])^{-1} \acsigbar[0] \acsigbar[0] \T \big)$ can then be rewritten as $\Tr(X X^T)$ with $X = \sqrt{(\acsig[\theta])^{-1}} ~ \acsigbar[0]$, which is a convex function composed with a linear function and is therefore convex. Finally, the last term is a positive definite quadratic form and is therefore convex in $\acmu[0]$.

\vspace{0.2cm}
\noindent{\sc Dynamics}. The state propagation equations follow from known properties of linear transformations of Gaussian random variables \citep{liu2019linear}. Since actions after index $0$ are entirely deterministic, we can express state uncertainty at future time steps $\stsigbar$ directly as linear functions of the initial action uncertainty $\acsigbar[0]$. This parallels results in the chance-constrained path planning literature \citep{blackmore2011chance}.

\vspace{0.2cm}
\noindent{\sc Safety Constraints}. The safety constraints for $\stsafe$ and $\stsafeterm$ can be handled similarly. Consider the chance constraint $\Pr\big[s \not \in \stsafe\big] < \epsilon$, with $s \sim \normal(\stmu, \stsigbar \stsigbar\T)$. Evaluating such a constraint using sampling would require prohibitively many samples for small $\epsilon$ and result in a nonconvex optimization problem. We instead leverage techniques from chance constrained optimization to represent this constraint deterministically. Let the polyhedral safe set be defined by $r$ linear inequalities as
\[
\stsafe = \bigcap_{i=1}^r ~ \{ \st[] \mid u_i \T \st[] \leq v_i \}.
\]
Deriving a tight closed-form expression for a joint constraint over multiple linear inequalities is a nontrivial problem that is typically handled by an approximation scheme \citep{cheng2012second}. We conservatively bound the probability of violating each inequality by $\epsilon / r$, noting that this implies
\[
\Pr ~ [\st[] \not \in S] \leq \sum_{i=1}^r \Pr ~ [u_i \T \st[] > v_i] \leq \sum_{i=1}^r \frac{\epsilon}{r} = \epsilon.
\]
We now aim to derive a closed-form counterpart for $r$ constraints of the form
\begin{align*}
    \Pr [u_i^T \st[] > v_i] \leq \frac{\epsilon}{r}, \quad ~ \quad \st[] \sim \normal(\stmu, \stsigbar \stsigbar\T).
\end{align*}
Since $\st[]$ is normally distributed, this constraint is equivalent to the deterministic constraint
\begin{align*}
    v_i - \stmu \T u_i \geq \Phi^{-1} \left( 1 - \frac{\epsilon}{r} \right) \| \stsigbar u_i \|_2,
\end{align*}
where $\Phi$ is the standard Gaussian CDF \citep{duchilecturenotes}. Each of our $r$ constraints now becomes a second-order cone constraint and can be handled by conventional convex optimization solvers. If the original problem is infeasible, we relax these constraints with slack variables which we linearly penalize in the objective.

\subsection{Policy Gradient with Safety Penalty}
\label{sec:policy}
We now modify the standard policy gradient formulation \eqref{eq:RLoptimization} to include a term penalizing corrections by the safety guide, effectively training the base policy to behave safely. Our objective becomes
\begin{align}\label{eq:rloptimizationppo}
    \max_{\theta} J^p(\pol), \quad ~ \quad ~ \quad J^p(\pol) = \E_{(s, a) \sim \polsafe} \Big[ \re(s, a) - 
    \beta ~ d\big(\polsafe(~ \cdot \mid s), \pol(~ \cdot \mid s) \big)\Big],
\end{align}
where $d$ is a positive definite statistical distance which is continuous in $\st[]$ for $\pol,\polsafe \in \Pol$ and $\beta>0$ is a regularization parameter. For notational convenience, the expectation draws from the stationary state distribution induced by $\polsafe$ and the associated action distribution. We show in Section~\ref{sec:basesafe} that any positive definite, continuous $d$ results in a safe base policy after training. We choose the squared $l_2$ parameter distance for its numerical properties:
\begin{align*}
    d\big(\polsafe(~ \cdot \mid s), \pol(~ \cdot \mid s) \big) \coloneqq
    \| \acmu[\textrm{safe}] - \acmu[\theta] \|_2^2
    + \| \acsig[\textrm{safe}] - \acsig[\theta] \|_2^2.
\end{align*}
We can now obtain our optimal parameters $\theta^*$ using gradient ascent on a Monte Carlo estimator similar to \eqref{eq:vanillapolicygrad} with an added term for the safety penalty.


\section{Theoretical Analysis}
\label{sec:analysis}
We show that our policy leads to safe exploration at training time with arbitrarily high probability. We then prove that coupling reward maximization with a safety penalty in \eqref{eq:rloptimizationppo} leads to a safe optimal base policy. This is highly desirable as it eliminates conflict between the base policy and the safety guide, mitigates distributional shift, and reduces the computational burden on the agent at test time.

\subsection{Training Time Safety}
Consider a standard episodic training setting where an episode terminates after a set number of time steps or upon violation of the state safety constraints.

\begin{proposition}
    Consider an arbitrary natural number $T$ and safety tolerance $\epsilon > 0$ from Algorithm~\ref{alg:guide}. Then over $T$ training steps, the expected number of states $\st$ such that $\st \not \in \stsafe$ is at most $\epsilon T$.
\end{proposition}

This follows directly from the constraints on the optimization problem in Algorithm~\ref{alg:guide}. Specifically, there is an $\epsilon$ chance of sampling an action from the safe distribution that leads to an unsafe state, in which case the episode ends in at most $H$ time steps. Assumption~\ref{ass:stable} guarantees that with probability $1 - \epsilon$ the action sampled will be safe and subsequent optimizations will remain feasible.

Since $\epsilon$ is a design parameter, this expectation can be driven to be arbitrarily small, at the cost of imposing additional conservatism in the exploration process. In practice, this quantity can be effectively set to zero by a small concession on the size of the safe sets $\stsafe$ and $\stsafeterm$. Shrinking these by some factor $1 - \delta$ gives the safe policy a buffer to the true unsafe region, allowing it to recover from unsafe actions by softening the chance inequality constraints in Algorithm~\ref{alg:guide}. Our experiments in Section~\ref{sec:experiments} use this technique to maintain \textit{perfect safety} over the course of a million training steps.

\subsection{Base Policy Safety}
\label{sec:basesafe}
In order to derive theoretical guarantees for the optimal policy of \eqref{eq:rloptimizationppo}, we introduce two assumptions.

\begin{assumption}
\label{ass:approx}
    The parameterized base policy class $\pol$ is a \textit{universal approximator}. Namely, for every policy $\poltarget \in \Pol$ and desired $\epsilon$, there exists a parameterized $\pol \in \Pol$ such that
    \begin{align*}
    \sup_{\st[], \ac[]} | \poltarget(\ac[] \mid \st[]) - \pol(\ac[] \mid \st[]) | < \epsilon.
    \end{align*}
\end{assumption}

\begin{assumption}
\label{ass:boundedrewards}
    The reward $\re(\tau)$ is bounded over all trajectories $\tau$.
\end{assumption}

Assumption~\ref{ass:approx} parallels a standard assumption in the deep learning literature that a richly parameterized network is arbitrarily expressive. Assumption~\ref{ass:boundedrewards} is similarly benign, and is immediately satisfied in a typical setting where rewards are bounded and trajectories are finite.


\begin{lemma}
\label{lem:approx}
    For every $\poltarget \in \Pol$ and $\epsilon_J > 0$, there exists a learned parameterization $\pol$ such that
    \[
        J(\poltarget) - J(\pol) < \epsilon_J,
    \]
    where $J(\pi) = \E_{\tau \sim \pi} \re(\tau)$ is the standard reinforcement learning objective.
\end{lemma}
\begin{proof}
    Let $p(\st[])$ be the initial state distribution and $p(\st[t+1] \mid \st[t], \ac[t])$ represent the environment transition dynamics. For notational simplicity, we define $\Delta J \defeq J(\poltarget) - J(\pol)$. Then we can write
    \begin{align}
        \Delta J = \int \re(\tau) p(\st[0]) \delta \pi(\tau) \prod_{t=0}^M p(\st[t+1] \mid \st[t], \ac[t]) d\tau, \label{eq:deltaJ}
    \end{align}
    where
    \[
        \delta \pi (\tau) = \prod_{t=0}^M \poltarget(\ac \mid \st) - \prod_{t=0}^M \pol(\ac \mid \st). 
    \]
    Assumption~\ref{ass:approx} implies that there exists a parameter vector $\theta$ such that $\delta \pi(\tau)$ can be bounded for all $\tau$ by an arbitrarily small quantity. Since $\re(\tau)$ in \eqref{eq:deltaJ} is bounded by Assumption~\ref{ass:boundedrewards} and probability distributions integrate to $1$, $\Delta J$ can be driven arbitrarily close to zero.
\end{proof}

Lemma~\ref{lem:approx} relates the universal approximation properties from Assumption~\ref{ass:approx} to the reward incurred by the policy. We now proceed with the main theoretical result.

\begin{theorem}
\label{thm:safebase}
    An optimal parameter vector $\theta^*$ which maximizes \eqref{eq:rloptimizationppo} is such that the base policy $\polstar$ is safe; i.e., the equality $\polstar(~ \cdot \mid \st) = \polsafestar(~ \cdot \mid \st)$ holds except on a set of measure zero with respect to the stationary state density function induced by $\polstar$ in a Radon-Nikodym sense.
\end{theorem}
\begin{proof}
    To prove by contradiction, assume that $\polstar$ is not safe. Define the set of states where the policy diverges from its safe representation as
    \[
        U = \{ \st[] : d(\polsafestar(~ \cdot \mid \st[]), \polstar(~ \cdot \mid \st[])) > 0 \} = \{s : d(s) > 0\}.
    \]
    with some abuse of notation. Now, consider the measure $\mu^*$ induced by the stationary state density function of $\polstar$. Since probability distributions integrate to one, $\mu^*$ is finite on compact sets; Euclidian space is also locally compact Hausdorff and second countable, and hence we have $\mu^*$ regular (Theorem 7.8 in \cite{folland1999real}).
    
    Since $d$ is continuous in $\st[]$ by assumption, $U$ is the inverse image of an open set under a continuous function and is therefore open. Regularity of $\mu^*$ and $\mu^*(U) > 0$ (by assumption) implies that there exists a compact set $\bar U \subset U$ such that $\mu^*(\bar U) > 0$. Since continuous functions attain their minimum over compact sets, we have that $d(s) > \delta$ for all $s \in \bar U$ for some $\delta > 0$.
    
    We now show that the difference in objectives between the safe and base policies is given by
    \[
        J^p(\polsafestar) - J^p(\polstar) \geq \beta \delta \mu^*(\bar U) > 0.
    \]
    Observe that the state action marginal in the expectation \eqref{eq:rloptimizationppo} is always taken with respect to $\polsafestar$; therefore, the reward terms vanish and the safety penalty is the only remaining term. By the previous discussion, this is at least $\delta \mu^*(\bar U)$, providing the desired expression.
    
    Finally, we invoke Lemma~\ref{lem:approx} to construct a policy $\pi_{\theta'} \in \Pol$ such that $J^p(\polsafestar) - J^p(\pi_{\theta'}) < \beta \delta \mu^*(\bar U) / 2$, noting that the safety penalty in \eqref{eq:rloptimizationppo} can be driven arbitrarily close to zero by Assumption~\ref{ass:approx} and continuity of $d$. This implies $J^p(\pi_{\theta'}) > J^p(\polstar)$, which is a contradiction.
\end{proof}

Theorem~\ref{thm:safebase} shows that the optimal parameters $\theta^*$ for our objective \eqref{eq:rloptimizationppo} produce a \textit{safe base policy} $\polstar$. Provided that gradient ascent effectively maximizes \eqref{eq:rloptimizationppo}, we can be confident that the policy has learned to behave safely and no longer requires the safety guide. This has three key advantages.

\begin{enumerate}
    \item \textit{Harmony between the base policy and safety guide.} Without a safety penalty, there is limited incentive for the base policy to learn to correct its own unsafe actions; the executed actions and ensuing rewards are always drawn from the action distribution of the safety guide. As noted in \cite{koller2018learningbased}, this decoupling can lead to a perpetual conflict between the base policy and the safety guide, with the base policy constantly approaching the boundaries of the safe set and the guide constantly correcting. Theorem~\ref{thm:safebase} shows that our method resolves this issue.
    \item \textit{Mitigation of distributional shift.} One potential concern with this method involves distributional shift; our policy gradient step updates the base policy, while rewards are sampled using the safe policy. Theorem~\ref{thm:safebase} implies that as training progresses, the distributional shift between these two policies decays to zero.
    \item \textit{Reduction of computational burden.} Solving the safety guide optimization problem requires significant computational effort. Theorem~\ref{thm:safebase} shows that the safety guide can be removed at test time without compromising safety. This can free up agent resources for other tasks.
\end{enumerate}

We note that in the setting where $J^p(\pol)$ is not completely maximized, the safety penalty $d(s)$ can still be concretely evaluated in any region of the state space. This provides the designer of the system with a quantitative measure of the level of safety of the base policy as well as insights into which regions of the state space are most dangerous.


\section{Numerical Experiments}
\label{sec:experiments}
\begin{figure}[tbp]
\floatconts
{fig:example}
{\vspace{-0.85cm}\caption{Experimental setup and results for the quadrotor setting. (a) The $\phi$-$\dot \phi$ plane of the quadrotor system safety sets. The dashed lines represent the true bounds of the terminal safety set $\stsafeterm$; we inner approximate this by a polytope. In practice, we also slightly shrink $\phi_{\min}$ and $\phi_{\max}$ by some factor $1-\delta$. (b) Test-time performances of a policy gradient agent trained with and without the safety guide on the double integrator task. The thick line indicates mean performance over five runs, with the shaded area representing the standard deviation. (c) Test-time average episode length. The policy trained with the safety guide achieves the maximum episode length of $250$ even when the safety guide is removed, indicating that the base policy has learned to behave safely.}\vspace{-0.2cm}}
{%
\subfigure[]{
\label{fig:safetyset}
\raisebox{0.48cm}{\resizebox{.3\textwidth}{!}{\begin{tikzpicture}[declare function={
        g(\x)=sqrt(-(\x-2))/2; 
        f(\x)=-sqrt((\x+2))/2; 
        gxmin=-2;gxmax=2;
        fxmin=-2;fxmax=2;}]
    
    \draw [<->, ultra thick] (-3,0) -- (3,0) node [right] {$\phi$};
    \draw [<->, ultra thick] (0,-2.2) -- (0,2.2) node [above] {$\dot \phi$};

    \draw[fill=green!60!black, opacity=0.05]  (-2, -2) -- (-2,2) -- (2,2) -- (2,-2) -- cycle;    
    \draw [black] (-2, -2) -- (-2, 2);
    \draw [black] (2, -2) -- (2, 2);
    
    \node (S) at (-1.5, 1.5) {$\stsafe$};
    
    \node [right] (xmax) at (2, -0.3) {$\phi_{\max}$};
    \node [left] (xmin) at (-2, -0.3) {$\phi_{\min}$};
    
    \draw[fill=green!60!black, fill opacity=0.15]  (-2, 0) -- (-2,1) -- (2,0) -- (2,-1) -- cycle;    
    \node (ST) at (-1.5, 0.5) {$\stsafeterm$};
    \draw[thick, dashed] plot[domain=gxmin:gxmax,samples=51,smooth] (\x,{g(\x)}); 
    \draw[thick, dashed] plot[domain=fxmin:fxmax,samples=51,smooth] (\x,{f(\x)}); 
    
\end{tikzpicture}}}
}
\quad 
\subfigure[]{%
\label{fig:results}
\includegraphics[width=0.3\textwidth]{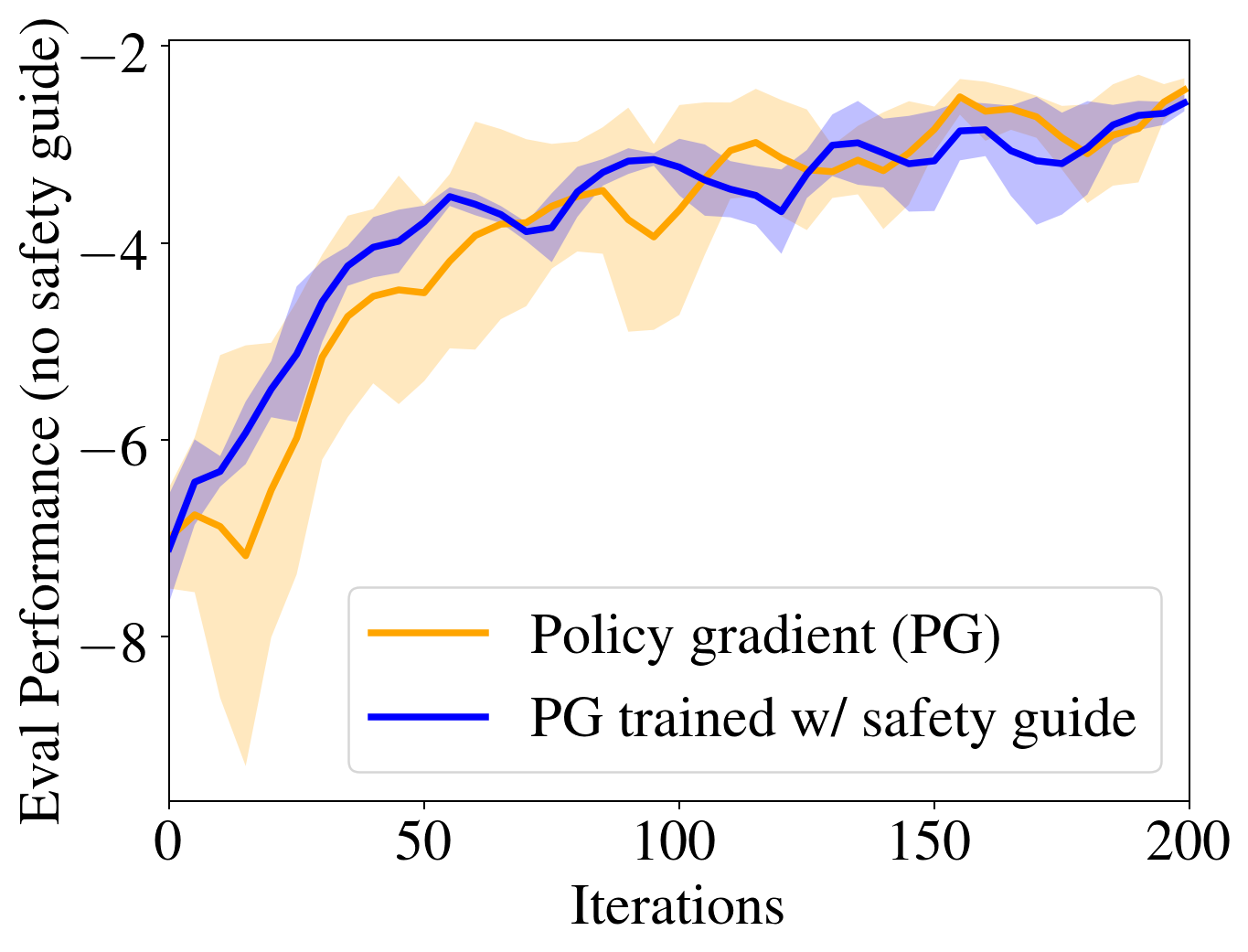}
}
\quad 
\subfigure[]{
\label{fig:length}
\includegraphics[width=0.3\textwidth]{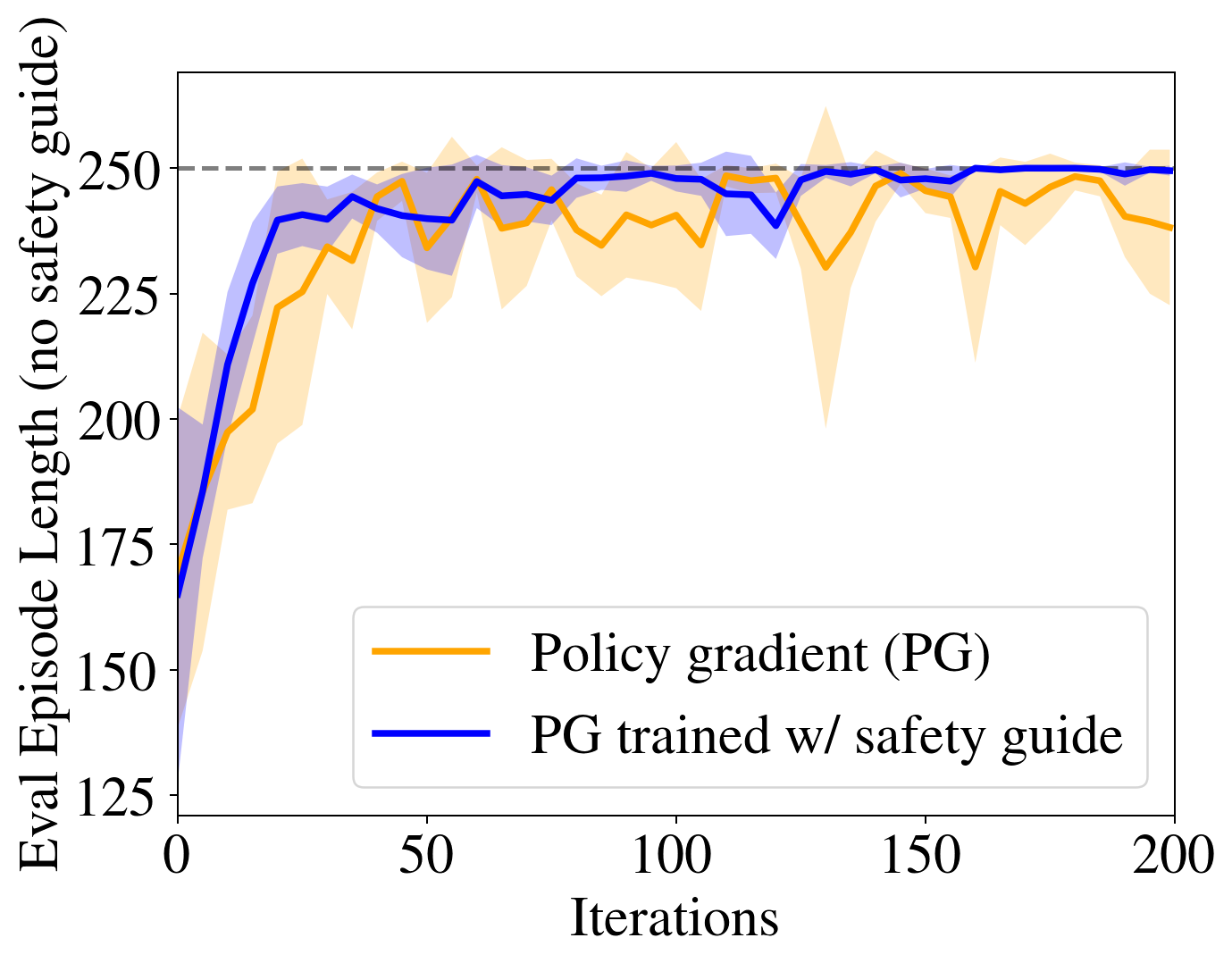}
}
}
\vspace{-0.1cm}
\end{figure}

Consider a two-dimensional quadrotor with state $\st = [x_t, \dot x_t, y_t, \dot y_t, \phi_t, \dot \phi_t]$, where $(x,y)$ is the quadrotor position and $\phi$ is the counter-clockwise angle to the vertical. The episode terminates if the quadrotor hits the ground or tilts more than $0.5$ radians. For early termination, the reward penalizes impact speed for hitting the ground ($r(\st) = -1 - 2 |\dot y_t|$) or rotational speed for excessive tilt ($r(\st) = -1 - 5 |\dot \phi_t|$). Otherwise, the quadrotor is incentivized to hover close to the ground while remaining centered horizontally ($r(\st) = -0.01 y_t - 0.01 |x_t|$). The control inputs are $\ac = [f_t, \tau_t]$, with $f_t \in [-2, 2]$ the vertical thrust and $\tau_t \in [-2, 2]$ the torque. Using the time step $\Delta t = 0.02$, we simulate the system using the following linearized dynamics about the hovering equilibrium
\[
    \begin{bmatrix}
        x_{t+1} \\ \dot x_{t+1} \\
        y_{t+1} \\ \dot y_{t+1} \\
        \phi_{t+1} \\ \dot \phi_{t+1}
    \end{bmatrix} =
    \begin{bmatrix}
        1 & 0 & 0 & \Delta t & 0 & 0 \\
        0 & 1 & 0 & 0 & \Delta t & 0 \\
        0 & 0 & 1 & 0 & 0 & \Delta t \\
        0 & 0 & -g \Delta t & 1 & 0 & 0 \\
        0 & 0 & 0 & 0 & 1 & 0 \\
        0 & 0 & 0 & 0 & 0 & 1 \\
    \end{bmatrix}
    \begin{bmatrix}
        x_t \\ \dot x_t \\
        y_t \\ \dot y_t \\
        \phi_t \\ \dot \phi_t
    \end{bmatrix}
    +
    \begin{bmatrix}
        0 & 0\\
        0 & 0\\
        0 & 0\\
        0 & 0\\
        \Delta t / m & 0\\
        0 & \Delta t / I\\
    \end{bmatrix}
    \begin{bmatrix}
    f_t \\ \tau_t
    \end{bmatrix}
\]
for mass $m=1$, inertia $I=1$, and gravity $g=1$. We design our safety set $\stsafe$ to have the constraints $y \geq 0.05$ and $-0.45 \leq \phi \leq 0.45$. Our terminal safe set $\stsafeterm$ consists of the same position bounds as well as a position-dependent velocity bound that captures the maximum velocity that can be brought to zero by the end of the corresponding safe set interval. Since this curve scales with the square root of distance, we inner approximate this by a polytope (Figure~\hyperref[fig:example]{2a}). The safety tolerance, planning horizon, and safety penalty are set as $\epsilon = 0.01$, $H = 15$, and $\beta = 1.5$. Our network consists of two hidden layers of size $64$ with $\tanh$ nonlinearities. We collected $5000$ steps per batch with an episode length of $250$ steps. The learning rate is $0.002$ and discount factor is $\gamma = 0.95$. Our reported results include the top $5$ of $10$ seeds by average eval performance---a common approach for mitigating policy initialization variance \citep{wu2017scalable}. The safety guide optimization problem is solved using MOSEK \citep{mosek}.

Our training approach achieved perfect safety over a training corpus of a million steps without compromising performance (Figure~\hyperref[fig:example]{2b}). Furthermore, Figure~\hyperref[fig:example]{2c} shows that safety guide-trained policy rapidly achieves the optimal average episode length of $250$ steps even when the safety guide is removed. This suggests that the safety penalty effectively induces the base policy to behave safely without having to try unsafe actions.
\section{Conclusion}
This work addresses the challenge of safe RL using a novel approach that combines a policy gradient agent with a chance-constrained MPC safety guide. The safety guide receives as input the proposed action distribution from the base policy and imposes additional safety requirements. By design, the safety guide intervenes minimally and modifies the base policy's proposed action distribution only if it inevitably leads towards an unsafe region of the state space. An additional safety penalty on these corrections in the overall objective allows us to provide theoretical guarantees that our base policy learns to behave safely without having to explore unsafe actions. We empirically justify our proposed method through numerical experiments on a linearized quadrotor control task.

\bibliography{main.bib}

\end{document}